\documentclass{article}

\PassOptionsToPackage{numbers, compress}{natbib}

\usepackage[preprint]{nips_2018}

\usepackage[utf8]{inputenc} 
\usepackage[T1]{fontenc}    
\usepackage{hyperref}       
\usepackage{url}            
\usepackage{booktabs}       
\usepackage{amsfonts}       
\usepackage{nicefrac}       
\usepackage{microtype}      
\usepackage{graphicx}
\usepackage{amsthm, amsmath, amssymb}
\usepackage{subcaption}
\usepackage{wrapfig}

\newtheorem{theorem}{Theorem}
\newtheorem{example}{Example}

\newtheorem{lemma}{Lemma}
\newtheorem*{definition}{Definition}

\def\psfancypar#1#2{\begingroup\def\par{\endgraf\endgroup\lineskiplimit=0pt}
               \setbox2=\hbox{\large\sc #2}
               \newdimen\tmpht \tmpht \ht2 \advance\tmpht by \baselineskip
               \font\hhuge=Times-Bold at \tmpht
               \setbox1=\hbox{{\hhuge #1}}
               \count7=\tmpht \count8=\ht1
               \divide\count8 by 1000 \divide\count7 by \count8
               \tmpht=.001\tmpht\multiply\tmpht by \count7
               \font\hhuge=Times-Bold at \tmpht
               \setbox1=\hbox{{\hhuge #1}}
               \noindent
                \hangindent1.05\wd1
               \hangafter=-2 {\hskip-\hangindent
               \lower1\ht1\hbox{\raise1.0\ht2\copy1}%
                \kern-0\wd1}\copy2\lineskiplimit=-1000pt}

\newcommand{\beq}{\begin{equation}}
\newcommand{\eeq}{\end{equation}}
\newcommand{\bqa}{\begin{eqnarray}}
\newcommand{\eqa}{\end{eqnarray}}
\newcommand{\bqn}{\begin{eqnarray*}}
\newcommand{\eqn}{\end{eqnarray*}}
\newcommand{\nn}{\nonumber}

\newcommand{\be}{\begin{enumerate}}
\newcommand{\ee}{\end{enumerate}}
\newcommand{\bi}{\begin{itemize}}
\newcommand{\ei}{\end{itemize}}
\newcommand{\bd}{\begin{description}}
\newcommand{\ed}{\end{description}}
\newcommand{\ba}{\begin{array}}
\newcommand{\ea}{\end{array}}
\newcommand{\bde}{\begin{definition}}
\newcommand{\ede}{\end{definition}}
\newcommand{\bex}{\begin{example}}
\newcommand{\eex}{\end{example}}

\def\boxit#1{\vbox{\hrule\hbox{\vrule\kern3pt
        \vbox{\kern3pt#1\kern3pt}\kern3pt\vrule}\hrule}}

\def\reals{ { {\rm  I \kern-0.15em R }  } }
\def\complex{ {\,{{\rm C} \kern-0.50em \raise0.20ex {  |}}\, }}

\def\0bf{{\bf 0}}
\def\1bf{{\bf 1}}
\def\2bf{{\bf 2}}
\def\3bf{{\bf 3}}
\def\4bf{{\bf 4}}
\def\5bf{{\bf 5}}
\def\6bf{{\bf 6}}
\def\7bf{{\bf 7}}
\def\8bf{{\bf 8}}
\def\9bf{{\bf 9}}

\def\Rbf{{\bf R}}

\newcommand{\A}{\mbox{${\cal A}$}}

\def\Rxx{\Rbf_{\ssstyle X\kern-.1em X}}

\let\ssstyle=\scriptscriptstyle

\def\Kout{\setbox1=\hbox{\Huge\bf K}\hbox to
1.05\wd1{\hspace{.05\wd1}
}}
\def\Sout{\setbox1=\hbox{\Huge\bf S}\hbox to 1.05\wd1{\hspace{.05\wd1}
}}

\theoremstyle{remark}

\title{Exponentially Consistent Kernel Two-Sample Tests}

\author{
 Shengyu Zhu\\
  Huawei Noah's Ark Lab\\
  Hong Kong, China\\
  \texttt{szhu05@syr.edu} \\
\And
  Biao Chen \\
  Syracuse University\\
  Syracuse, NY, USA\\
\texttt{bichen@syr.edu}\\
\And
 Zhitang Chen\\
Huawei Noah's Ark Lab\\
Hong Kong, China\\
\texttt{chenzhitang2@huawei.com} \\
}

\begin{document}

\maketitle

	\begin{abstract} 
	Given two sets of independent samples from unknown distributions $P$ and $Q$, a two-sample test decides whether to reject the null hypothesis that $P=Q$. Recent attention has focused on kernel two-sample tests as the test statistics are easy to compute, converge fast, and have low bias with their finite sample estimates. However, there still lacks an exact characterization on the asymptotic performance of such tests, and in particular, the rate at which the type-II error probability decays to zero in the large sample limit. In this work, we establish that a class of kernel two-sample tests are exponentially consistent with Polish, locally compact Hausdorff sample space, e.g., $\mathbb R^d$. The obtained exponential decay rate is further shown to be optimal among all two-sample tests satisfying the level constraint, and is independent of particular kernels provided that they are bounded continuous and characteristic. Our results gain new insights into related issues such as fair alternative for testing and kernel selection strategy. Finally, as an application, we show that a kernel based test achieves the optimal detection for off-line change detection in the nonparametric setting.
\end{abstract} 

\section{Introduction}
Given two sets of i.i.d.~samples, the two-sample problem decides whether or not to accept the null hypothesis that the generating distributions are the same, without imposing any parametric assumptions. This is important to a variety of applications, including data integration in bioinformatics \cite{Borgwardt2006}, statistical model criticism \cite{Kim2016Crit, Lloyd2015ModelCrit}, and training deep generative models \cite{Dziugaite2015TrainingGenerative,Li2017mmdGAN,Li2015Generative,Sutherland2017GeneandCrit}. Typical two-sample tests are constructed based on some distance measures between distributions, such as classical Kolmogorov-Smirnov distance \cite{Friedman1979multivariate}, Kullback-Leibler divergence (KLD) \cite{Bu2016KLestimation,Nguyen2010},  and maximum mean discrepancy (MMD), a reproducing kernel Hilbert space norm of the difference between kernel mean embeddings of distributions \cite{ Gretton2012,Gretton2012OptKernelLarge, Muandet2017KME, SimSch16Kernel,Zaremba2013Btest}. Notably, kernel based test statistics possess several key advantages such as computational efficiency and fast convergence, thereby attracting much attention recently.

A hypothesis test is usually evaluated by characterizing its type-II error probability subject to a level constraint on the type-I error probability. In this respect, existing kernel two-sample tests have been shown to be consistent, in the sense that the type-II error probability decreases to zero as sample sizes scale to infinity. While consistency is a desired property, quantifying how fast the error probability decays is even more desirable, as it provides a natural metric for comparing test performance. However, exact characterization on the decay rate is still elusive, even for some well-known kernel two-sample tests.  For example, assuming $n$ samples in both sets, the decay rate of the biased quadratic-time test in \cite{Gretton2012} is claimed to be (at least) $\mathcal O(n^{-0.5})$, 
based on a large deviation bound on the test statistic. The large deviation bound has been observed to be loose in general, indicating that the above decay rate is loose too. Other works such as \cite{Gretton2012OptKernelLarge,Sutherland2017GeneandCrit,Zaremba2013Btest} have established the limiting distributions of the test statistics, but they also do not give a tight decay rate. Clearly, no statistical optimality can be claimed if the characterization itself is loose.


More recently, in the context of goodness of fit testing, \citet{Zhu2018UHT} showed that the quadratic-time kernel two-sample tests have the type-II error probability vanishing exponentially fast at a rate determined by the KLD between the two generating distributions. A strong condition for this result is that sample sizes need scale in different orders. Their approach, however, is not readily applicable when sample sizes increase in the same order, e.g., when the two sets have an equal number of samples.  This is because existing Sanov's theorems only hold for the sample sequence originating from {{one}} given distribution, whereas the acceptance region defined by the kernel two-sample test involves {{two}} sample sequences from {{different}} distributions. As such, the key seems to be an extended version of Sanov's theorem that handles two distributions; this is not apparent as existing tools, e.g., Cram{\' e}r theorem \cite{Dembo2009} that is used for proving Sanov's theorem, can only deal with a single distribution. 



 
The first goal of this paper is to seek an exact statistical characterization for a widely used kernel two-sample test. We establish an extended version of Sanov's theorem w.r.t.~the topology induced by a pairwise weak convergence of probability measures. Our proof is inspired by \citet{Csiszar2006simple} which proved  original Sanov's theorem of one sample sequence in the $\tau$-topology. Based on the idea of \cite{Zhu2018UHT}, we then show that the biased quadratic-time kernel two-sample test in \cite{Gretton2012} is exponentially consistent when sample sizes scale in the same order. The obtained exponential decay rate depends only on the generating distributions and the samples sizes under the alternative hypothesis, and is further shown to be the {{optimal}} one among all two-sample tests satisfying the  level constraint. A notable implication is that kernels affect only the sub-exponential term in the type-II error probability, provided that they are bounded continuous and characteristic. We also comment that the extended Sanov's theorem may be of independent interest and may be applied to other large deviation applications.

Our second goal is to derive an optimality criterion for nonparametric two-sample tests as well as a way of finding more tests achieving this optimality. Towards this goal, we characterize the maximum exponential decay rate for any two-sample test under the given level constraint. Furthermore, a sufficient condition is derived for the type-II error probability to decay at least exponentially fast with the maximum exponential rate (possibly violating the level constraint). These results provide new insights into related issues such as fair alternative for testing and kernel selection strategy, which are elaborated in Sections~\ref{sec:discussion} and~\ref{sec:exp}. As an application, we apply our results to the off-line change detection problem and show that a kernel based test achieves the optimal detection in terms of the exponential decay rate of the type-II error probability. To our best knowledge, this is the first time that a test is shown to be optimal for detecting the presence of a change  in the nonparametric setting.

In Section~\ref{sec:pre}, we briefly review the MMD and the two-sample testing. In  Section~\ref{sec:mainresults}, we present our main results on the exact and optimal exponential decay rate for a class of kernel two-sample tests, followed by discussions on related issues. We apply our results to off-line change detection in Section~\ref{sec:app} and conduct synthetic experiments in Section~\ref{sec:exp}. Section~\ref{sec:conclusion} concludes the paper.

\section{Maximum mean discrepancy, two-sample testing, and test threshold}
\label{sec:pre}
We briefly review the MMD and its weak metrizable property. We then describe the two-sample problem as statistical hypothesis testing and choose a suitable threshold for the level constraint.

\paragraph{Maximum mean discrepancy} Let $\mathcal F$ be a reproducing kernel Hilbert space (RKHS) defined on a topological space $\mathcal X$ with reproducing kernel $k$. Let $x$ be an $\mathcal X$-valued random variable with probability measure $P$, and $\mathbf E_{x}f(x)$ the expectation of $f(x)$ for a function $f:\mathcal X\to\mathbb R$. Assume that $k$ is bounded continuous. Then for every Borel probability measure $P$ defined on $\mathcal X$, there exists a unique element $\mu_k(P)\in\mathcal F$ such that $\mathbf E_{x}f(x)=\langle f, \mu_k(P)\rangle_{\mathcal F}$ for all $f\in\mathcal F$ \cite{Berlinet2011RKHS}. The MMD between two Borel probability measures $P$ and $Q$ is the RKHS-distance between $\mu_k(P)$ and $\mu_k(Q)$, which can be expressed as 
\[d_k(P,Q)=\|\mu_k(P)-\mu_k(Q)\|_{\mathcal F}=\left(\mathbf E_{x,x'}k(x,x')+\mathbf{E}_{y,y'}k(y,y')-2\mathbf E_{x,y}k(x,y)\right)^{1/2},\]
where $x,x'$ i.i.d.~$\sim P$ and $y,y'$ i.i.d.~$\sim Q$ \cite{Gretton2012}. If the kernel $k$ is characteristic, then $d_k(P,Q)=0$ if and only if $P=Q$ \cite{Sriperumbudur2010hilbert}. This property enables the MMD to distinguish different distributions.

We present a weak metrizable property of $d_k$, which will be used to establish our main results in Section~\ref{sec:mainresults}. Let $\mathcal P$ denote the set of all Borel probability measures defined on $\mathcal X$. For a sequence of probability measures $P_l\in\mathcal P$, we say that $P_l\to P$ weakly if and only if $\mathbf E_{x\sim P_l} f(x)\to\mathbf E_{x\sim P} f(x)$ for every bounded continuous function $f:\mathcal X\to\mathbb R$. 

\begin{theorem}[{\cite{SimSch16Kernel, Sriperumbudur2016EstPM}}]
	\label{thm:MMDmetrize}
The MMD $d_k(\cdot,\cdot)$ metrizes the weak convergence on $\mathcal P$ if the following two conditions hold:
{\bf (A1)} the sample space $\mathcal X$ is Polish, locally compact and Hausdorff; {\bf (A2)} the kernel $k$ is bounded continuous and characteristic.
\end{theorem}	

We note that the weak metrizable property is also favored for training deep generative models \cite{Arjovsky2017WGAN,Li2017mmdGAN}.  An example of Polish, locally compact Hausdorff space is $\mathbb R^d$, and the Gaussian kernel satisfies  the conditions of \hyperref[thm:MMDmetrize]{\bf A2}.


\paragraph{Two-sample testing based on the MMD} Let $x^n$ and $y^m$ be independent samples, with $x^n\sim P$ and $y^m\sim Q$ where $P$ and $Q$ are unknown. The two-sample testing is to decide between $H_0:P=Q$ and $H_1:P\neq Q$. Let $\hat P_n$ and $\hat Q_m$ be the respective empirical measures of $x^n$ and $y^m$, that is, $\hat{P}_n=\frac{1}{n}\sum_{i=1}^n\delta_{x_i}$ and $\hat Q_m=\frac{1}{m}\sum_{i=1}^m\delta_{y_i}$ with $\delta_x$ being Dirac measure at $x$. Then the squared MMD can be estimated by 
\begin{align}
d_k^2(\hat P, \hat Q)=\frac{1}{n^2}\sum_{i=1}^n\sum_{j=1}^nk(x_i,x_j)+\frac{1}{m^2}\sum_{i=1}^m\sum_{j=1}^mk(y_i,y_j)-\frac{2}{nm}\sum_{i=1}^n\sum_{j=1}^mk(x_i,y_j),\nn
\end{align}
which is a biased statistic originally proposed in \cite{Gretton2012}. A hypothesis test for the two-sample testing can then be constructed by comparing this statistic with a threshold $\gamma_{n,m}$: if $d_k(\hat P_n, \hat Q_m)\leq \gamma_{n,m}$, then the test accepts the null hypothesis $H_0$. The acceptance region is hence defined as  $\mathcal A(n,m)=\{(x^n,y^m):d_k(\hat P_n,\hat Q_m)\leq \gamma_{n,m}\}$. There are two types of errors: a type-I error is made if $(x^n,y^m)\notin\mathcal A(n,m)$ despite $H_0:P=Q$ being true, and a type-II error occurs when $(x^n,y^m)\in\mathcal A(n,m)$ under $H_1:P\neq Q$. The type-I and type-II error probabilities are given by
\begin{align}
\alpha_{n,m} &= \mathbf P_{x^ny^m}\left((x^n,y^m)\notin \mathcal A(n,m)\right)~\text{under}~ H_0:P=Q,\nn\\
\beta_{n,m} &= \mathbf P_{x^ny^m}\left((x^n,y^m)\in\mathcal A(n,m)\right)~\text{under}~H_1:P\neq Q,\nn
\end{align}
respectively. Bear in mind that $\alpha_{n,m}$ and $\beta_{n,m}$ are computed w.r.t.~the true yet unknown distributions. 

With a carefully chosen threshold, the above kernel test has been shown  to be consistent in \cite{Gretton2012}. That is, $\beta_{n,m}\to 0$ as $n,m\to\infty$, while $\alpha_{n,m}\leq\alpha$ with $\alpha\in(0,1)$ being set in advance. In this paper, we study the exponential decay rate of $\beta_{n,m}$ in the large sample limit, subject to the same level constraint. Specifically, we aim to characterize 
\[\liminf_{n,m\to\infty}-\frac{1}{n+m}\log\beta_{n,m},~\text{subject to}~\alpha_{n,m}\leq\alpha.\]
The above limit is also called the type-II error exponent in information theory \cite{Cover2006}. If the limit is positive, then the test is said to be  exponentially consistent.

\paragraph{A suitable threshold}
We directly use a result from \citep{Gretton2012} in order to pick a proper threshold for the level constraint $\alpha_{n,m}\leq\alpha$. Such tests are referred to as level $\alpha$ tests in statistics \cite{Casella2002}.
\begin{lemma}[{\cite[Theorem~7]{Gretton2012}}]
	\label{lem:Gretton}
	Let $P$, $Q$, $x^n$, $y^m$, $\hat P_n$, $\hat Q_m$ be defined as in Section~\ref{sec:pre}, and assume \hyperref[thm:MMDmetrize]{\bf A2} with $K$ being a positive upper bound on $k(\cdot,\cdot)$. Then under the null hypothesis $H_0:P=Q$, 
	\begin{align}
	\mathbf{P}_{x^ny^m}\left(d_k(\hat P_n,\hat Q_m)>2(K/m)^{\frac{1}{2}}+2(K/n)^{\frac{1}{2}}+\epsilon\right)
	\leq 2\exp\left(-\frac{\epsilon^2mn}{2K(m+n)}\right).\nn
	\end{align}
\end{lemma} 
Therefore, for a given $0<\alpha<1$, choosing
\begin{align}
\label{eqn:gammanm}
\gamma_{n,m}=\left((K/m)^{\frac{1}{2}}+(K/n)^{\frac{1}{2}}\right)\left(2+\sqrt{2\log(2\alpha^{-1})}\right),
\end{align}
the kernel test $d_k(\hat {P}_n,\hat Q_m)\leq\gamma_{n,m}$ has its type-I error probability $\alpha_{n,m}<\alpha$, hence is a level $\alpha$ test.
\section{Main results}  
\label{sec:mainresults}
In this section, we present our main results on the type-II error exponent of a class of kernel two-sample tests. The first and the most important step is to establish an extended Sanov's theorem that works with two sample sequences. 
\subsection{Extended Sanov's theorem}
We define a pairwise weak convergence: we say $(P_l,Q_l)\to(P,Q)$ weakly if and only if both $P_l\to P$ and $Q_l\to Q$ weakly. We consider $\mathcal P\times\mathcal P$ endowed with the topology induced by this pairwise weak convergence.  It can be verified that this topology is equivalent to the product topology on $\mathcal P\times\mathcal P$ where each $\mathcal P$ is endowed with the topology of weak convergence. An extended version of Sanov's theorem is given below. 

\begin{theorem}[Extended Sanov's Theorem] 
Let $\mathcal X$ be a Polish space, $x^n$~i.i.d.~$\sim P$, and $y^m$~i.i.d.~$\sim Q$. Assume $0<\lim_{n,m\to\infty}\frac{n}{n+m}=c<1$. Then for a set $\Gamma\subset\mathcal P\times\mathcal P$, it holds that 
\begin{align}
\limsup_{n,m\to\infty}-\frac{1}{n+m}\log\mathbf {P}_{x^ny^m}((\hat{P}_n,\hat{Q}_m)\in\Gamma)&\leq	\inf_{(R,S)\in\operatorname{int}\Gamma} cD(R\|P)+(1-c)D(S\|Q),
	\nn\\
 \liminf_{n,m\to\infty}-\frac{1}{n+m}\log\mathbf{P}_{x^ny^m} ((\hat{P}_n,\hat{Q}_m)\in\Gamma)\nn
	&\geq \inf_{(R,S)\in\operatorname{cl}\Gamma} cD(R\|P)+(1-c)D(S\|Q),\nn
	\end{align}
where $\operatorname{int}$ and $\operatorname{cl}$ denote the interior and closure w.r.t.~the pairwise weak convergence, respectively.
\end{theorem}	
We prove the above result in finite sample space and then extend it to general Polish space, with two simple combinatorial lemmas as prerequisites. See details in   Appendix~\ref{sec:extendedSanov}. 


\subsection{Exact exponent of type-II error probability}
\label{sec:ExpConsistency}
With the extended Sanov's theorem and a vanishing threshold $\gamma_{n,m}$ given in Eq.~(\ref{eqn:gammanm}), we are ready to establish the exponential decay of the type-II error probability. Our result follows.
\begin{theorem}
	\label{thm:mainresult1}
 Assume \hyperref[thm:MMDmetrize]{\bf A1}, \hyperref[thm:MMDmetrize]{\bf A2}, and  $\lim_{n,m\to\infty}\frac{n}{n+m}=c\in(0,1)$. Under the alternative hypothesis $H_1:P\neq Q$, also~assume~that \[0<D^*:=\inf_{R\in\mathcal P} cD(R\|P) + (1-c)D(R\|Q)^*<\infty.\] Given $0<\alpha<1$, the kernel test
	$d_k(\hat P_n, \hat Q_m)\leq \gamma_{n,m}$
	with $\gamma_{n,m}$ in Eq.~(\ref{eqn:gammanm}) is an exponentially consistent level $\alpha$ test: 
	\[\alpha_{n,m}\leq\alpha,~\text{and}~\liminf_{n,m\to\infty}-\frac1{n+m}\log\beta_{n,m}=D^*.\]
\end{theorem}
\begin{proof} 
We use the fact that testing if $(x^n,y^m)\in\{(x^n,y^m):d_k(\hat P_n,\hat Q_m)\leq\gamma_{n,m}\}$ is equivalent to testing if $(\hat P_n,\hat Q_m)\in\{(P', Q'):d_k(P',Q')\leq\gamma_{n,m}\}$. Since the threshold $\gamma_{n,m}\to 0$ as $n,m\to\infty$, $\gamma_{n,m}$ is eventually smaller than any fixed $\gamma>0$, and hence $\{(P', Q'):d_k(P',Q')\leq\gamma_{n,m}\}\subset\{(P', Q'):d_k(P',Q')\leq\gamma\}$ for large enough $n,m$. By the extended Sanov's theorem, the type-II error probability decays at least exponentially fast if $(P,Q)\notin\operatorname{cl}(\{(P', Q'):d_k(P',Q')\leq\gamma\})$, which can be satisfied by picking $\gamma<d_k(P,Q)$ under $H_1:P\neq Q$ and using the weak convergence property of the MMD (cf.~Theorem~\ref{thm:MMDmetrize}). We then show that the exponential decay rate is both lower bounded and upper bounded by $D^*$ based on the lower semi-continuity of the KLD \cite{VanErven2014RenyiKLD} and Stein's lemma \cite{Dembo2009}, respectively. Details can be found in  Appendix~\ref{sec:proofmainresult}.
\end{proof}

Therefore, when $0<c<1$, the type-II error probability  vanishes as $\mathcal O(e^{-(n+m)(D^*-\epsilon)})$, where $\epsilon\in(0,D^*)$ is fixed and can be arbitrarily small. The result also shows that kernels only affect the sub-exponential term in the type-II error probability, provided that they meet the conditions of \hyperref[thm:MMDmetrize]{\bf A2}.

Not covered in Theorem~\ref{thm:mainresult1} is the case when $n$ and $m$ scale in different orders, i.e., when $c=0$ or $1$. Without loss of generality, we may consider only $c=1$, with $\lim_{n,m\to\infty}{n}/{m}\to\infty$. If $0<D(P\|Q)<\infty$ under the alternative hypothesis, then  \cite[Theorem~4]{Zhu2018UHT} indicates that
\begin{align}
\label{eqn:diff1}
\liminf_{n,m\to\infty}-\frac{1}{m}\log\beta_{n,m}=D(P\|Q),
\end{align}
which leads to a degenerate result on the error exponent w.r.t.~the sample size $n+m$:
\begin{align}
\label{eqn:difforders}
\liminf_{n,m\to\infty}-\frac{1}{n+m}\log\beta_{n,m}= \liminf_{n,m\to\infty}\frac{1}{1+\frac{n}{m}}\left(-\frac{1}{m}\log\beta_{n,m}\right)=0.\nn
\end{align}
Notice that, with $c=1$ (and $0$) we have $D^*=0$. Then Theorem~\ref{thm:mainresult1} still holds if we remove the assumption $c\in(0,1)$. However, the error exponent being $0$ also includes the case where $\beta_{n,m}$ is bounded away from $0$. The more insightful perspective is to look at Eq.~(\ref{eqn:diff1}), and the test is said to be exponentially consistent w.r.t.~the sample size $m$.






\subsection{Optimal exponent and more exponentially consistent two-sample tests}
\label{sec:moreexp}
We can identify other two-sample tests that are at least exponentially consistent based on the above results. In particular, the lower bounds still hold if another test has a smaller type-II error probability, or if $\mathbf P_{x^ny^m}(\A'(n,m))\leq \mathbf P_{x^ny^m}(\A(n,m))$ under $H_1:P\neq Q$, where $\mathcal A'(n,m)$ is the acceptance region defined by the test. A special case is considered in the following theorem, directly from Theorem~\ref{thm:mainresult1} and Eq.~(\ref{eqn:diff1}).
\begin{theorem}
	\label{thm:lowerbd}
	Let $\mathcal X$, $x^n$, $y^m$, $P$, $Q$, $\hat P$, $\hat Q$, and $D^*$ be defined as in Theorem~\ref{thm:mainresult1}. Assume \hyperref[thm:MMDmetrize]{\bf A1} and \hyperref[thm:MMDmetrize]{\bf A2}. Let $\A'(n,m)$ be the acceptance region of another two-sample test and $\beta'_{n,m}$ the type-II error probability. If $\A'(n,m)\subset\{(x^n,y^m):d_k(\hat P_n,\hat Q_m)\leq\gamma'_{n,m}\}$ where $\gamma'_{n,m}\to 0$ as $n,m\to\infty$, then
	\[\liminf_{n,m\to\infty}-\frac{1}{n+m}\log\beta'_{n,m}\geq D^*,\]
	when $0<\lim_{n,m\to\infty}\frac{n}{n+m}=c<1$ and $0<D^*<\infty$; and 
	\[\liminf_{n,m\to\infty}-\frac{1}{m}\log\beta'_{n,m}\geq D(P\|Q),\]
	when $\lim_{n,m\to\infty}\frac nm=\infty$ and $0<D(P\|Q)<\infty$.
	
\end{theorem}	

The above theorem characterizes only the type-II error exponent. A suitable threshold is needed to guarantee the test be level $\alpha$ for practical use. Our next result provides an upper bound on the optimal type-II error exponent of any (asymptotically) level $\alpha$ test.

\begin{theorem}
	\label{thm:upperbd}
	Let $x^n$, $y^m$, $P$, $Q$, and $D^*$ be defined as in Theorem~\ref{thm:lowerbd}.  For a  test $\A'(n,m)$ which is (asymptotically) level $\alpha,0<\alpha<1$, its type-II error probability $\beta'_{n,m}$ satisfies \[\liminf_{n,m\to\infty}-\frac{1}{n+m}\log\beta'_{n,m}\leq D^*,\]
	if~$0<\lim_{n,m\to\infty}\frac{n}{n+m}=c<1$ and $0<D^*<\infty$; and 
	\[\liminf_{n,m\to\infty}-\frac{1}{m}\log\beta'_{n,m}\leq D(P\|Q),\]
if $\lim_{n,m\to\infty}\frac nm=\infty$ and $0<D(P\|Q)<\infty$.
\end{theorem}	
\begin{proof} Let $P'$ be such that $cD(P'\|P)+(1-c)D(P'\|Q)=D^*$ for $0<c<1$. Define $A_n=\{x^n:|\frac{1}{n}\log\frac{dP'(x^n)}{dP(x^n)}- D(P'\|P)|\leq\epsilon\}$, and
$B_m= \{y^m:|\frac{1}{m}\log\frac{dP'(y^m)}{dQ(y^m)}-D(P'\|Q)|\leq\epsilon\}$, where $\epsilon>0$ is fixed and can be arbitrary. Here ${dP'}/{dP}$ and ${dP'}/{dQ}$ are Radon-Nikodym derivatives and exist by the finiteness of $D^*$. Consider the acceptance region $A_n\times B_m\cap\A'(n,m)$, from which we can obtain an upper bound $D^*+\epsilon$ on the type-II error exponent of $\A'(n,m)$. Since $\epsilon$ can be arbitrarily small, then $D^*$ is an upper bound on the type-II error exponent. When $c=1$, we can set $P'=P$ and apply the above argument; alternatively, we may compare the test with the optimal goodness-of-fit test in \cite{Zhu2018UHT} and use Stein's lemma \cite{Dembo2009} to establish the upper bound. See Appendix~\ref{sec:optimality} for details. 
\end{proof}

This theorem shows that the kernel test $d_k(\hat P_n,\hat Q_m)\leq \gamma_{n,m}$ is an optimal level $\alpha$ two-sample test, by choosing the type-II error exponent as the asymptotic performance metric. Moreover, Theorems~\ref{thm:lowerbd} and \ref{thm:upperbd} together provide a way of finding more asymptotically optimal two-sample tests:
\begin{itemize}
	\item An unbiased estimator of the squared MMD, denoted by $\operatorname{MMD}_u^2$, is also proposed in \cite{Gretton2012}. The test $\operatorname{MMD}_u^2\leq(4K/\sqrt{n})\sqrt{\log(\alpha^{-1})}$ is a level $\alpha$ test, assuming $n=m$. As $k(\cdot,\cdot)$ is finitely bounded by $K$, we have $|\mathrm{MMD}_u^2-\mathrm{MMD}_b^2|\leq2K/n$ and the acceptance region of the unbiased test is a subset of $\operatorname{MMD}_b^2\leq(4K/\sqrt{n})\sqrt{\log(\alpha^{-1})}+2K/n$. Then its type-II error probability vanishes exponentially at a rate of $\inf_{R\in\mathcal P}\frac{1}{2}D(R\|P)+\frac12D(R\|Q)$.
	
	
	\item It is also possible to consider a family of kernels for the test statistic \cite{Fukumizu2009,Sriperumbudur2016EstPM}. For a given family $\kappa$, the test statistic is $\sup_{k\in\kappa} d_k(\hat P_n, \hat Q_m)$ which also metrizes weak convergence under suitable conditions, e.g., when $\kappa$ consists of finitely many Gaussian kernels \cite[Theorem~3.2]{Sriperumbudur2016EstPM}. If $K$ remains to be an upper bound for all $k\in\kappa$, then comparing $\sup_{k\in\kappa} d_k(\hat P_n, \hat Q_m)$ with $\gamma_{n,m}$ in Eq.~(\ref{eqn:gammanm}) results in an asymptotically optimal level $\alpha$ test.
\end{itemize}
\subsection{Discussions}
\label{sec:discussion}
\paragraph{Fair alternative} In \cite{Ramdas2015}, a notion of fair alternative is proposed for two-sample testing as dimension increases, which is to fix $D(P\|Q)$ under the alternative hypothesis for all dimensions. This idea is guided by the fact that the KLD is a fundamental information-theoretic quantity  determining the hardness of hypothesis testing problems. This approach, however, does not take into account the impact of sample sizes. In light of our results, perhaps a better choice is to fix $D^*$ in Theorem~\ref{thm:mainresult1} when the sample sizes grow in the same order. In practice, $D^*$ may be hard to compute, so fixing its upper bound $(1-c)D(P\|Q)$ and hence $D(P\|Q)$ is reasonable. 


\paragraph{Kernel choice} The main results indicate that the type-II error exponent is independent of kernels as long as they are bounded continuous and characteristic. We remark that this indication does not contradict previous studies on kernel choice, as the sub-exponential term can dominate in the finite sample regime. In light of the exponential consistency, it then raises interesting connections with a kernel selection strategy, where part of samples are used as training data to choose a kernel and the remaining samples are used with the selected kernel to compute the test statistic \cite{Gretton2012OptKernelLarge,Sutherland2017GeneandCrit}. On the one hand, the sample size should not be too small so that there are enough data for training. On the other hand, if the number of samples is large enough and the exponential decay term becomes dominating, directly using the entire samples may be good enough to have a low type-II error probability, provided that kernel is not too poor. This point will be further illustrated by experiments in Section~\ref{sec:exp}.

\paragraph{Threshold choice} As also discussed in \cite{Zhu2018UHT}, the distribution-free threshold, $\gamma_{n,m}$ in Eq.~(\ref{eqn:gammanm}), is  loose in general \cite{Gretton2012}. In practice, the threshold can be computed based on some estimate of the null distribution from the given samples, such as a bootstrap procedure and using the eigenspetrum of the Gram matrix on the aggregate sample \cite{Gretton2009,Gretton2012}. While these approaches can meet the level constraint in the large sample limit, they however bring additional randomness on the threshold and further on the type-II error probability. Similar to \cite{Zhu2018UHT},  we can take the minimum of such a threshold and the distribution-free one to achieve the optimal type-II error exponent, while the type-I error constraint holds in the asymptotic sense, i.e., $\lim_{n,m\to\infty}\alpha_{n,m}\leq\alpha$.

\paragraph{Other discrepancy measures} Other distance measures between distributions may also metrize the weak convergence on $\mathcal P$, such as L\'evy-Prokhorov metric, bounded Lipschitz metric, and Wasserstein distance. If we directly compute such a distance between the empirical measures and compare it with a decreasing threshold, the obtained test would have the same optimal type-II error exponent as in Theorem~\ref{thm:lowerbd}.  However, unlike Lemma~\ref{lem:Gretton} for the MMD based statistic, there does not exist a uniform or distribution-free threshold such that the level constraint is satisfied for all sample sizes. Similar to the kernel Stein discrepancy based goodness-of-fit test in \cite{Zhu2018UHT}, a possible remedy  is to relax  the level constraint to an asymptotic one, but a uniform characterization on the decay rate of the estimated distance is still required. We will not expand into this direction, because computing such distance measures from samples is generally more costly than the MMD based statistics. 

\section{Application to off-line change detection}
\label{sec:app}
In this section, we apply our results to the off-line change detection problem. 

Let $z_1,\ldots,z_{n}\in\mathbb{R}^d$ be an independent sequence of observations. Assume that there is at most one change-point at index $1<t<n$, which, if exists, indicates that $z_i\sim P, 1\leq i\leq t$ and $z_i\sim Q, t+1\leq i\leq n$ with $P\neq Q$. The off-line change-point analysis consists of two steps: 1) detect if there is a change-point in the sample sequence; 2) estimate the index $t$ if such a change-point exists. Notice that a method may readily extend to multiple change-point and on-line settings, through sliding windows running along the sequence, as in \cite{Desobry2005onlinekernel, Harchaoui2009KernelChange,Li2015Changedetection}.

The first step in the change-point analysis is usually formulated as a hypothesis testing problem: 
\begin{align}
H_0:~&z_i \sim P, i=1,\ldots,n, \nn\\
H_1:~&\text{there exists}~1<t<n~\text{such that}\nn\\
&z_i\sim P, 1\leq i\leq t~\text{and}~z_i\sim Q\neq P, t+1\leq i\leq n.\nn
\end{align}
Let $\hat P_{i} $ and $\hat Q_{n-i}$ denote the empirical measures of sequences $z_1,\ldots, z_i$ and $z_{i+1},\ldots,z_n$, respectively. Then an MMD based test can be directly constructed using the  maximum partition strategy:
\[\text{decide}~H_0,~\text{if}~\max_{a_n\leq i\leq b_n}d_k(\hat P_{i}, \hat Q_{n-i})\leq\gamma_n,\]
where the maximum is searched in the interval $[a_n, b_n]$ with $a_n > 1$ and $b_n < n$. If the test favors $H_1$, we can proceed to estimate the change-point index by $\operatorname{argmax}_{a_n\leq i\leq b_n} d_k(\hat P_i, \hat Q_{n-i})$. Here we characterize the performance of detecting the presence of a change for this test, using Theorems~\ref{thm:mainresult1} and \ref{thm:upperbd}. We remark that the assumptions on the search interval and on the change-point index in the following theorem are standard practice in this setting \cite{Basseville1993detectionchanges,Desobry2005onlinekernel, Harchaoui2009KernelChange,James1987testsforchange,Li2015Changedetection}.

\begin{theorem}
\label{thm:changedetection}
Let $a_n/n\to u>0$ and $b_n/n\to v <1$ as $n\to\infty$. Under the alternative hypothesis $H_1$, assume that the change-point index $t$ satisfies $u<\lim_{n\to\infty}t/n=c<v$, and that $0< D^*<\infty$ where $D^*$ is defined in Theorem~\ref{thm:mainresult1}. Further assume that the kernel $k$ satisfies \hyperref[thm:MMDmetrize]{\bf A2}, with $K>0$ being an upper bound. Given $0<\alpha<1$, set $c_{\min}=\min\{a_n(n-a_n), b_n(n-b_n)\}$ and $\gamma_n=\sqrt{{2K}/{a_n}}+\sqrt{2K/b_n}+\sqrt{2Kn\log(2n\alpha^{-1})/{c_{\min}}}$. Then the test $\max_{a_n\leq i\leq b_n}d_k(\hat P_{i}, \hat Q_{n-i})\leq\gamma_n$ is level $\alpha$ and also achieves the optimal type-II error exponent, that is, 
\[\alpha_{n}\leq \alpha,~\text{and}~\liminf_{n\to\infty}-\frac{1}{n}\log \beta_n=D^*,\]
where $\alpha_{n}$ and $\beta_n$ are the type-I and type-II error probabilities, respectively.
\end{theorem}
\begin{proof}
Since $\mathbf P_{z^n}(\max_{a_n\leq i\leq b_n}d_k(\hat P_{i}, \hat Q_{n-i})>\gamma_n)\leq \sum_{a_n\leq i\leq b_n}\mathbf P_{z^n}(d_k(\hat P_{i}, \hat Q_{n-i})>\gamma_n)$, it suffices to make each $\mathbf P_{z^n}(d_k(\hat P_{i}, \hat Q_{n-i})>\gamma_n)\leq \alpha/n$ under the null hypothesis $H_0$. This can be verified using Lemma~\ref{lem:Gretton} with the choice of $\gamma_n$  in the above theorem. To see the optimal type-II error exponent, consider a simpler problem where the possible change-point $t$ is known, i.e., a two-sample problem between $z_1,\ldots,z_t$ and $z_{t+1},\ldots,z_n$. Since $\gamma_n\to0$ as $n\to\infty$, applying Theorems~\ref{thm:mainresult1} and \ref{thm:upperbd} establishes the optimal type-II error exponent.
\end{proof}
\section{Experiments}
\label{sec:exp}
This section presents empirical results to validate our previous findings. We begin with a toy example to demonstrate the exponential consistency, and then consider how kernel choice and sample sizes affect the type-II error probability. We set equal sample sizes, i.e., $n=m$, and pick the significance level $\alpha=0.05$ in all experiments.

\paragraph{Exponential consistency} 
\begin{wrapfigure}{r}{0.4\textwidth}
	\vspace{-1.5em}
	\includegraphics[width=\linewidth]{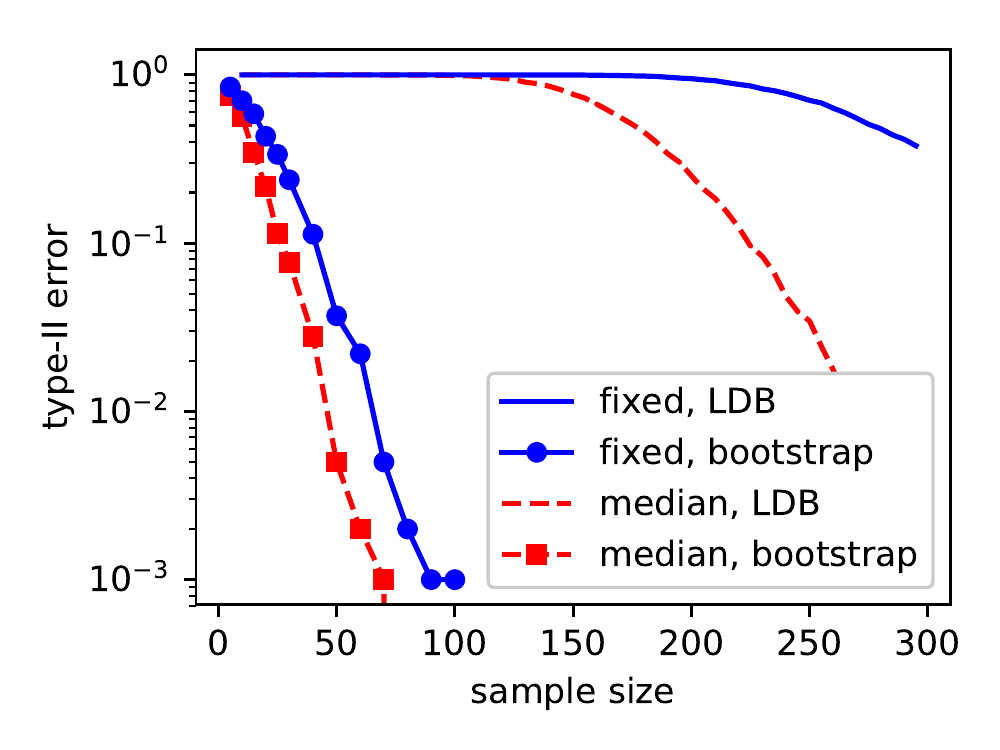}
	\caption{2D Gaussians with different means.} 
	\label{fig:pes_simple}
\end{wrapfigure} 
While there have been various experiments on the type-II error probability, the exponential decay behavior  has been scarcely reported. To this end, we perform a simple experiment and display the type-II error probability in the logarithm scale. Let $x^n~\text{i.i.d.}~\sim \mathcal{N}({\mu}_P, {I})$ and $y^m~\text{i.i.d.}~\sim \mathcal{N}({\mu}_Q, {I})$, where $\mu_P=[0.25, 0.25]^T$, ${\mu}_Q=[1,1]^T$, and $I$ is the $2\times2$ identity matrix. We use the biased test statistic $d_k(\hat P_n,\hat Q_m)$ with Gaussian kernel $k(x,y)=\exp(-\|x-y\|_2^2/w)$. A fixed choice of $w=1$ and the median heuristic are employed for the kernel bandwidth.  We also consider two threshold choices: one is from the Large Deviation Bound (LDB), given in Eq.~(\ref{eqn:gammanm}); and the other is from a bootstrap method in \cite{Gretton2012}, with $1000$ bootstrap replicates. We repeat $1000$ trials and report the result in Figure~\ref{fig:pes_simple}.

We observe that all the type-II error probabilities exhibit an exponential decay rate as the sample number increases. The LDB threshold is quite conservative and the error probability starts decaying with much more samples. Although the main theorems in Section~\ref{sec:mainresults} do not include the median bandwidth, the figure shows that it also leads to an exponential decay of the type-II error probability. This might be because the median distance lies within a small neighborhood of some fixed bandwidth in this experiment, hence behaving similarly.



\paragraph {Kernel choice vs.~Sample size} Following the discussions in Section~\ref{sec:discussion}, we investigate how kernel choice and sample number affect the test performance. We consider Gaussian kernels that are determined by their bandwidths. \citet{Sutherland2017GeneandCrit} use part of samples as training data to select the bandwidth, which we call the trained bandwidth. The estimated MMD is then computed using the trained bandwidth and the remaining samples. 

For the first experiment, we take a similar setting from  \cite{Sutherland2017GeneandCrit}: $P$ is a $3\times 3$ grid of 2D standard normals, with spacing $10$ between the centers; $Q$ is laid out identically, but with covariance $\frac{\epsilon-1}{\epsilon+1}$ between the
coordinates. Here we pick $\epsilon=6$ and generate $n=m=720$ samples from each distribution. We pick splitting ratios $r=0.25$ and $r=0.5$ for computing the trained bandwidth. Correspondingly, there are $n=m=540$ and $n=m=360$ samples used to calculate the test statistic, respectively. For each case with $n=m\in\{360, 540, 720\}$, we report in Figure~\ref{fig:pes2_3x3} the type-II error probabilities over different bandwidths, averaged over $200$ trials. The unbiased test statistic $d_u^2(\hat P_n,\hat Q_m)$ is used and the test threshold is obtained using bootstrap with $500$ permutations. We also mark the trained bandwidths corresponding to the respective sample sizes in the figure (red star marker). 

Figure~\ref{fig:pes2_3x3} verifies that the trained bandwidth is close to the optimal one in terms of the type-II error probability. Moreover, it indicates that a large range of bandwidths lead to lower or comparable error probabilities if we directly use the entire samples for testing. As the sample number increases, the exponential decay term in the type-II error probability becomes dominating and the effect of kernel choice diminishes. However, since the desired range of bandwidths  is not known in advance, an interesting question is when we shall split data for kernel selection and what is a proper splitting ratio.


\begin{figure}[t!]
\centering
\begin{minipage}[t]{0.33\linewidth}
	\includegraphics[width=\linewidth]{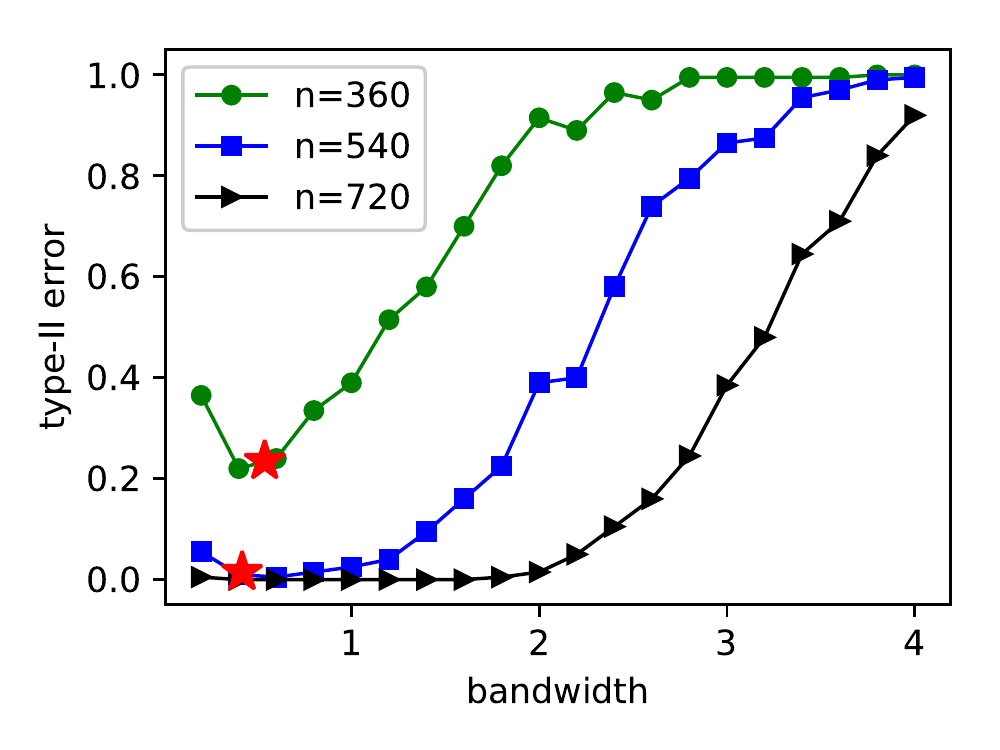}
	\subcaption{}
	\label{fig:pes2_3x3}
\end{minipage}%
\begin{minipage}[t]{0.33\linewidth}
	\includegraphics[width=\linewidth]{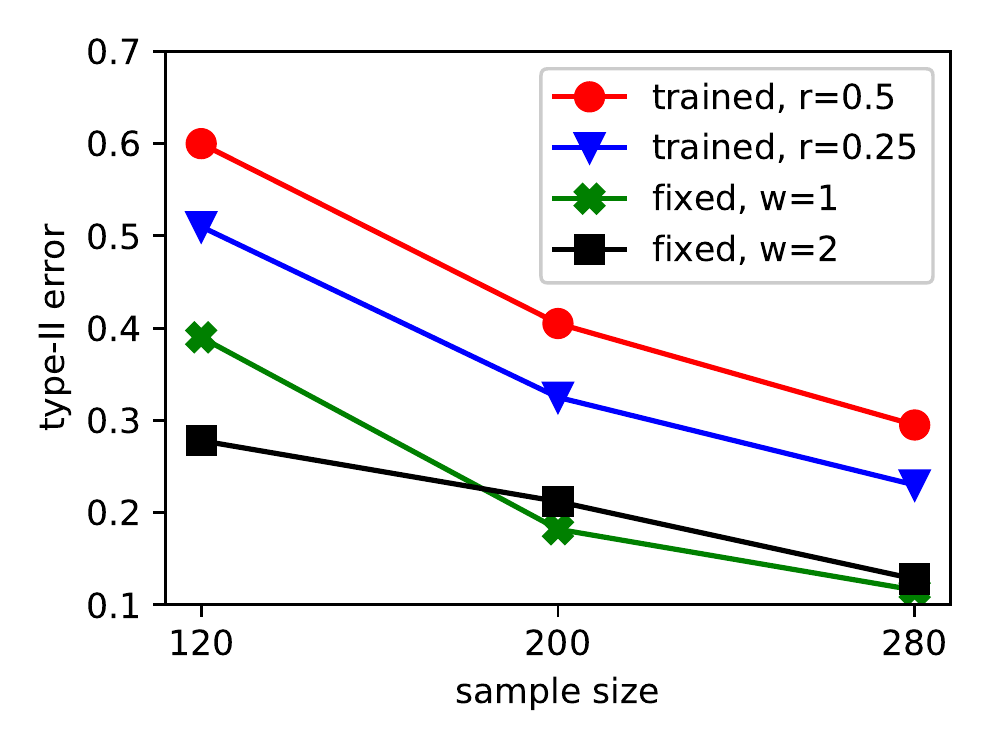}
	\subcaption{}
	\label{fig:pes2_MOG}
\end{minipage}%
\begin{minipage}[t]{0.33\linewidth}
	\includegraphics[width=\linewidth]{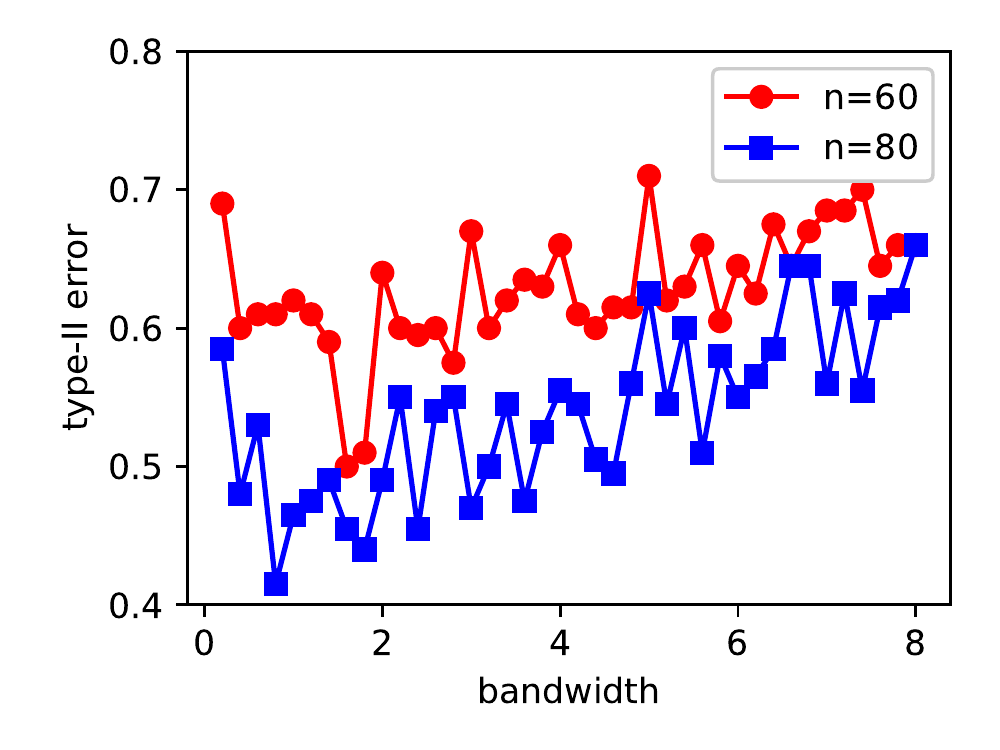}
	\subcaption{}
	\label{fig:pes2_MOG_grid}
\end{minipage}
\caption{Kernel choice vs.~Sample size. {\bf (a)} $3\times 3$ grid of 2D standard normals. Red star denotes the trained bandwidth.  {\bf (b-c)} 1D Gaussian mixture.}
\end{figure}

In the second experiment, we directly use the setup in \cite{Liu2016GoodnessFit}. We draw $x^n~\text{i.i.d.}\sim\sum_{k=1}^5 a_k\mathcal N(\mu_k,\sigma^2)$ with $a_k=1/5$, $\sigma^2=1$, and $\mu_k\sim\operatorname{Uniform}[0,10]$, and then generate $y^m$ by adding standard Gaussian noise (perturbation) to $\mu_k$. We consider splitting ratios $r=0.25$ and $r=0.5$ of the entire samples used as training data and compute $d_k^2(\hat P_n,\hat Q_m)$ based on the rest samples. For comparison, the kernel tests with fixed bandwidths $w=1$ and $w=2$ are also evaluated, which estimate the MMD based on the entire samples. All the test  thresholds are computed using bootstrap with $500$ replicates. We repeat $500$ trials and report the type-II error probabilities in Figure~\ref{fig:pes2_MOG}. It shows that the more samples we use to compute the test statistic, the lower type-II error probability we get; in other words, kernel choice is less important than the sample size for this setting. This point is further illustrated in Figure~\ref{fig:pes2_MOG_grid} where we show the type-II error probabilities of $n=m=60$ and $n=m=80$ samples over different kernel bandwidths. The kernel selection strategy in \cite{Sutherland2017GeneandCrit} does not perform well in this experiment, which also motivates future studies on when to use such a kernel selection strategy.

\section{Conclusion}
\label{sec:conclusion}
In this paper, a class of kernel two-sample tests are shown to exponentially consistent and to attain the optimal type-II error exponent, provided that kernels are bounded continuous and characteristic. A notable implication is that kernels affect only the sub-exponential term in the type-II error probability. We apply our results to off-line change detection and show that a test achieves the optimal detection in the nonparametric setting. Finally, we empirically investigate how kernel choice and sample size affect the test performance. 
\newpage
{\small 
\bibliography{SZhuBib}
\bibliographystyle{abbrvnat}}
\newpage
\appendix
\section*{Appendix}

\section{Proof of the extended Sanov's theorem}
\label{sec:extendedSanov}
We first prove the result with a finite sample space and then extend it to the case with general Polish space. The prerequisites are two combinatorial lemmas that are standard tools in information theory.

For a positive integer $t$, let $\mathcal P_n(t)$ denote the set of probability distributions defined on $\{1,\ldots, t\}$ of form $P=\left(\frac{n_1}n,\cdots,\frac{n_t}n\right)$, with integers $n_1,\ldots, n_t$. Stated below are the two lemmas.
\begin{lemma}[{\cite[Theorem~11.1.1]{Cover2006}}]
	\label{lem:numEmpDistribution}
	$|\mathcal P_n(t)|\leq(n+1)^t.$
\end{lemma}
\begin{lemma}[{\cite[Theorem~11.1.4]{Cover2006}}]
	\label{lem:typeprob}
	Assume $x^n$ i.i.d.~$\sim Q$ where $Q$ is a distribution defined on $\{1,\ldots,t\}$. For any $P\in\mathcal P_n(t)$, the probability of the empirical distribution $\hat P_n$ of $x^n$ equal to $P$ satisfies	
	\[(n+1)^{-t}e^{-nD(P\|Q)}\leq \mathbf P_{x^n}(\hat P_n=P)\leq e^{-nD(P\|Q)}.\]
\end{lemma}
\subsection{Finite sample space}	
\paragraph{Upper bound} Let $t$ denote the cardinality of $\mathcal X$. Without loss of generality, assume that $\inf_{(R,S)\in\operatorname{int}\Gamma} cD(R\|P)+(1-c) D(S\|Q)<\infty$. Hence, the open set $\operatorname{int}\Gamma$ is non-empty. As $0<\lim_{n,m\to\infty}\frac{n}{n+m}=c<1$, we can find $n_0$ and $m_0$ such that there exists $(P'_n,Q'_m)\in\operatorname{int}\Gamma\cap P_n(t)\times P_m(t)$ for all $n>n_0$ and $m>m_0$, and that $cD(P_n'\|P)+(1-c)D(Q_m'\|Q)\to\inf_{(R,S)\in\operatorname{int}\Gamma} cD(R\|P)+(1-c) D(S\|Q)$ as $n,m\to\infty$. Then we have, with $n>n_0$ and $m>m_0$,
\begin{align}
\mathbf{P}_{x^ny^m}((\hat P_n, \hat Q_m)\in\Gamma) 
&=\sum_{(R,S)\in\Gamma\,\cap\, \mathcal{P}_{n}(t)\times\mathcal P_m(t)} \mathbf{P}_{x^ny^m}(\hat P_n=R, \hat Q_m=S)\nn\\
&\geq\sum_{(R,S)\in\operatorname{int}\Gamma\,\cap\, \mathcal{P}_{n}(t)\times\mathcal P_m(t)} \mathbf{P}_{x^ny^m}(\hat P_n=R, \hat Q_m=S)\nn\\
&\geq \mathbf{P}_{x^ny^m}(\hat P_n=P_n', \hat{Q}_m=Q_m')\nn\\
&= \mathbf P_{x^n}(\hat P_n=P'_n)\,\mathbf P_{y^m}(\hat Q_m=Q'_m)\nn\\
&\geq(n+1)^{-t}(m+1)^{-t}e^{-nD(P_n'\|P)}e^{-mD(Q_m'\|Q)}\nn,
\end{align}
where the last inequality is from Lemma~\ref{lem:typeprob}. It follows that \begin{align}
&~~~~\limsup_{n,m\to\infty}-\frac{1}{n+m}\log\mathbf{P}_{x^ny^m}((\hat{P}_n,\hat{Q}_m)\in\Gamma)\nn\\&\leq \lim_{n,m\to\infty}\frac1{n+m}\left(-t\log((n+1)(m+1))+nD(P'_n\|P)+mD(Q'_m\|Q)\right)\nn\\
&=\lim_{n,m\to\infty}\frac1{n+m}\left(nD(P'_n\|P)+mD(Q'_m\|Q)\right)\nn\\
&=\inf_{(R,S)\in\operatorname{int}\Gamma} cD(R\|P)+(1-c) D(S\|Q).\nn
\end{align}

\paragraph{Lower bound}
\begin{align}
\mathbf{P}_{x^ny^m}((\hat P_n, \hat Q_m)\in\Gamma) 
&= \sum_{(R,S)\in\Gamma\cap \mathcal{P}_{n}(t)\times\mathcal P_m(t)} \mathbf{P}_{x^n}(\hat P_n=R)\,\mathbf{P}_{y^m}(\hat Q_m=S)\nn\\
&\stackrel{(a)}{\leq} \sum_{(R,S)\in\Gamma\cap \mathcal{P}_n(t)\times\mathcal{P}_m(t)} e^{-nD(R\|P)}e^{-mD(S\|Q)}\nn\\
&\stackrel{(b)}{\leq} (n+1)^{t} (m+1)^{t}\sup_{(R,S)\in\Gamma} e^{-nD(R\|P)}e^{-mD(S\|Q)},
\end{align}
where $(a)$ and $(b)$ are due to  Lemma~\ref{lem:typeprob} and Lemma~\ref{lem:numEmpDistribution}, respectively. This gives \[\liminf_{n\to\infty}-\frac{1}{n+m}\log \mathbf{P}_{x^ny^m}((\hat{P}_n,\hat{Q}_m)\in\Gamma)\geq \inf_{(R,S)\in\Gamma} cD(R\|P)+(1-c)D(S\|Q),\]
and hence the lower bound by noting that $\Gamma\in\operatorname{cl}\Gamma$. Indeed, when the right hand side is finite, the infimum over $\Gamma$ equals the infimum over $\operatorname{cl}\Gamma$ as a result of the continuity of KLD for finite alphabets.

\subsection{Polish sample space}
We consider the general case with $\mathcal X$ being a Polish space. Now $\mathcal{P}$ is the space of probability measures on $\mathcal X$ endowed with the topology of weak convergence. To proceed, we introduce another topology on $\mathcal P$ and an equivalent definition of the KLD.

{\it $\tau$-topology: } denote by $\Pi$ the set of all partitions $\mathcal A=\{A_1,\ldots, A_t\}$ of $\mathcal X$ into a finite number of measurable sets $A_i$. For $P\in\mathcal P$, $\mathcal A\in\Pi$, and $\zeta>0$, denote 
\begin{align}
\label{eqn:opentautoplogy}
U(P,\mathcal A, \zeta) = \{P'\in\mathcal P:|P'(A_i)-P(A_i)|<\zeta, i=1,\dots,t\}.
\end{align}
The $\tau$-topology on $\mathcal P$ is the coarsest topology in which the mapping $P\to P(F)$ are continuous for every measurable set $F\subset\mathcal X$. A base for this topology is the collection of the sets (\ref{eqn:opentautoplogy}). We will use $\mathcal P_\tau$ when we refer to $\mathcal P$ endowed with this $\tau$-topology, and write the interior and closure of a set $\Gamma\in\mathcal P_\tau$ as $\operatorname{int}_\tau\Gamma$
and $\operatorname{cl}_\tau\Gamma$, respectively. We remark that the $\tau$-topology is stronger than the weak topology: any open set in $\mathcal P$ w.r.t.~weak topology is also open in $\mathcal P_\tau$ (see more details in \cite{Csiszar2006simple,Dembo2009}). The product topology on $\mathcal P_\tau\times\mathcal P_\tau$ is determined by the base of the form of 
\[U(P,\mathcal A_1, \zeta_1)\times U(Q, \mathcal A_2, \zeta_2),\]
for $(P,Q)\in\mathcal P_\tau\times\mathcal P_\tau$, $\mathcal A_1,\mathcal A_2\in\Pi$, and $\zeta_1,\zeta_2>0$. We still use $\operatorname{int}_{\tau}(\Gamma)$ and $\operatorname{cl}_{\tau}(\Gamma)$ to denote the interior and closure of a set $\Gamma\subset\mathcal P_\tau\times\mathcal P_\tau$. As there always exists $\mathcal A\in\Pi$ that refines both $\mathcal A_1$ and $\mathcal A_2$, any element from the base has an open subset \[\tilde{U}(P,Q,\mathcal A,\zeta):=U(P,\mathcal A, \zeta)\times U(Q, \mathcal A, \zeta)\subset\mathcal P_\tau\times\mathcal P_\tau,\]
for some $\zeta >0$. 

{\it Another definition of the KLD:} an equivalent definition of the KLD will also be used:
\begin{align}
D(P\|Q)=\sup_{\mathcal A\in\Pi} \sum_{i=1}^t P(A_i)\log\frac{P(A_i)}{Q(A_i)}=\sup_{\mathcal A\in\Pi}D(P^{\mathcal A}\|Q^{\mathcal A}),\nn
\end{align}
with the conventions $0\log 0=0\log\frac{0}{0}=0$ and $a\log\frac{a}{0}=+\infty$ if $a>0$. Here $P^{\mathcal A}$ denotes the discrete probability measure $(P(A_1),\ldots,P(A_t))$ obtained from probability measure $P$ and partition $\mathcal A$. It is not hard to verify that for $0<c<1$,
\begin{align}
\label{eqn:KLDdef}
cD(R\|P)+(1-c)D(S\|Q)&=c\sup_{\mathcal{A}_1\in\Pi}D(R^{\mathcal A_1}\|P^{\mathcal A_1})+(1-c)\sup_{\mathcal A_2\in\Pi}D(S^{\mathcal A_2}\|Q^{\mathcal A_2})\nn\\
&=\sup_{\mathcal{A}\in\Pi}\left(cD\left(R^{\mathcal A}\|P^{\mathcal A}\right)+(1-c)D\left(S^{\mathcal A}\|Q^{\mathcal A}\right)\right),
\end{align}
due to the existence of $\mathcal{A}$ that refines both $\mathcal{A}_1$ and $\mathcal A_2$ and the log-sum inequality \cite{Cover2006}.

We are ready to show the extended Sanov's theorem with Polish space.

\paragraph{Upper bound}
It suffices to consider only non-empty open $\Gamma$. If $\Gamma$ is open in $\mathcal P\times\mathcal P$, then $\Gamma$ is also open in $\mathcal P_\tau\times\mathcal P_\tau$. Therefore, for any $(R,S)\in\Gamma$, there exists a finite (measurable) partition $\mathcal A= \{A_1,\ldots,A_t\}$ of $\mathcal X$ and $\zeta>0$ such that 
\begin{align}
\label{eqn:opensubset}
\tilde{U}(R,S,\mathcal A,\zeta)=
\left\{(R',S'):|R(A_i)-R'(A_i)|<\zeta,|S(A_i)-S'(A_i)|<\zeta,i=1,\ldots,t\right\}\subset\Gamma.
\end{align}

Define the function $T:\mathcal X\to\{1,\ldots,t\}$ with $T(x)=i$ for $x\in A_i$. Then $(\hat P_n, \hat Q_m)\in\tilde{U}(R,S,\mathcal A,\zeta)$ with $R,S\in\Gamma$ if and only if the empirical measures $\hat P^{\circ}_n$ of $\{T(x_1),\ldots, T(x_n)\}:=T(x^n)$ and $\hat Q^{\circ}_m$ of $\{T(y_1),\ldots, T(y_m)\}:=T(y^m)$ lie in 
\[U^{\circ}(R,S,\mathcal A, \zeta)=\{(R^\circ,S^{\circ}):|R^{\circ}(i)-R(A_i)|<\zeta, |S^\circ(i)-S(A_i)|<\zeta,i=1,\ldots, t\}\subset \mathbb R^t\times\mathbb R^t.\]
Thus, we have
\begin{align}\mathbf{P}_{x^ny^m}((\hat P_n, \hat Q_m)\in\Gamma)&\geq\mathbf{P}_{x^ny^m}((\hat P_n, \hat Q_m)\in\tilde{U}(R,S,\mathcal A, \zeta))\nn\\
&=\mathbf{P}_{T(x^n)T(y^m)}((\hat P_n^{\circ}, \hat Q_m^{\circ})\in U^{\circ}(R,S,\mathcal A, \zeta)).\nn
\end{align}
As $T(x)$ and $T(y)$ takes values from a finite alphabet and $U^{\circ}(R,S,\mathcal A, \zeta)$ is open, we obtain that 
\begin{align}
&~\limsup_{n\to\infty}-\frac{1}{n+m}\log\mathbf{P}_{x^ny^m}((\hat P_n,\hat Q_m)\in\Gamma)\nn\\\leq&~ \limsup_{n\to\infty}-\frac{1}{n+m}\log\mathbf{P}_{T(x^n)T(y^m)}((\hat P_n^{\circ},\hat Q_m^{\circ})\in U^{\circ}(R,S,\mathcal A, \zeta))\nn\\
\leq&~\inf_{(R^\circ,S^\circ)\in U^{\circ}(R,S,\mathcal A, \zeta)} cD(R^\circ\|P^{\mathcal A})+(1-c)D(S^\circ\|Q^{\mathcal A})\nn\\
=&~\inf_{(R',S')\in\tilde{U}(R,S,\mathcal A, \zeta)} cD(R'^{\mathcal A}\|P^{\mathcal A})+(1-c)D(S'^{\mathcal A}\|Q^{\mathcal A})\nn\\
\leq&~ cD(R\|P)+(1-c)D(S\|Q),
\end{align}
where we have used definition of KLD in Eq.~(\ref{eqn:KLDdef})  and $(R,S)\in\tilde{U}(R,S,\mathcal A, \zeta)$ in the last inequality.  As $(R,S)$ is arbitrary in $\Gamma$, the lower bound is established by taking infimum over $\Gamma$.

\paragraph{Lower bound} With notations
\[\Gamma^{\mathcal A}=\{(R^{\mathcal{A}},S^\mathcal{A}):(R,S)\in\Gamma\},~ \Gamma(\mathcal A)=\{(R,S):(R^{\mathcal  A},S^{\mathcal{A}})\in\Gamma^{\mathcal{A}}\},\]
where $\mathcal A=\{A_1,\ldots,A_t\}$ is a finite partition, it holds that
\begin{align}
&~\mathbf{P}_{x^ny^m}((\hat P_n,\hat Q_m)\in\Gamma)\nn\\
\leq&~\mathbf P_{x^ny^m}((\hat P_n,\hat Q_m)\in\Gamma({\mathcal{A}}))\nn\\
=&~\mathbf P_{x^ny^m}((\hat P_n^{\mathcal A},\hat Q_m^{\mathcal A})\in\Gamma^{\mathcal A} \cap\mathcal{P}_{n}(t)\times{\mathcal P_m}(t))\nn\\
\leq&~(n+1)^t(m+1)^t\max_{(R^\circ,S^\circ)\in\Gamma^{\mathcal A}\cap\mathcal{P}_{n}(t)\times{\mathcal P_m(t)}}\mathbf P_{x^ny^m}\left(\hat{P}_n=R^\circ,\hat{Q}_m=S^\circ\right)\nn\\
\leq&~(n+1)^t(m+1)^t \exp\left(-\inf_{(R,S)\in\Gamma}  \left(nD(R^{\mathcal A}\|P^\mathcal{A})+mD(S^{\mathcal A}\|Q^\mathcal{A})\right)\right),\nn
\end{align}
where the last two inequalities are from Lemmas~\ref{lem:numEmpDistribution} and \ref{lem:typeprob}. As the above holds for any $\mathcal A\in\Pi$, Eq.~(\ref{eqn:KLDdef}) indicates
\begin{align}
&~\limsup_{n\to\infty}\frac{1}{n+m}\log\mathbf P_{x^ny^m}((\hat P_n,\hat Q_m)\in\Gamma)\nn\\
\leq&~\inf_{\mathcal{A}}\left(-\inf_{(R,S)\in\Gamma}  \left(cD(R^{\mathcal A}\|P^\mathcal{A})+(1-c)D(S^{\mathcal A}\|Q^\mathcal{A})\right)\right)\nn\\
=&~-\sup_{\mathcal{A}}\inf_{(R,S)\in\Gamma}  cD(R^{\mathcal A}\|P^\mathcal{A})+(1-c)D(S^{\mathcal A}\|Q^\mathcal{A}).\nn
\end{align}
Then the remaining of obtaining the lower bound is to show  
\[\sup_{\mathcal{A}}\inf_{(R,S)\in\Gamma} cD(R^{\mathcal A}\|P^\mathcal{A})+(1-c)D(S^{\mathcal A}\|Q^\mathcal{A})\geq \inf_{(R,S) \in\operatorname{cl}\Gamma}  cD(R\|P)+(1-c)D(S\|Q).\]




Assuming, without loss of generality, that the left hand side is finite, we only need to show

\[\operatorname{cl}\Gamma\cap B(P,Q,\eta)\neq\varnothing,\]
whenever \[\eta>\sup_{\mathcal{A}}\inf_{(R,S)\in\Gamma} cD(R^{\mathcal A}\|P^\mathcal{A})+(1-c)D(S^{\mathcal A}\|Q^\mathcal{A}).\] Here $B(P,Q,\eta)$ is the divergence ball defined as follows
\[B(P,Q,\eta)=\left\{(R,S):cD(R\|P)+(1-c)D(S\|Q)\leq\eta\right\},\]
which is compact in $\mathcal P\times\mathcal P$~w.r.t.~the weak topology, due to the lower semi-continuity of $D(\cdot\|P)$ and $D(\cdot\|Q)$ as well as the fact that $0<c<1$.

To this end, we first show the following:
\begin{align}
\label{eqn:clGamma}
\operatorname{cl}\Gamma=\bigcap_{\mathcal{A}}\operatorname{cl}\Gamma(\mathcal{A}).
\end{align}
The inclusion is obvious since $\Gamma\in\Gamma(\mathcal{A})$. The reverse means that if $(R,S)\in \operatorname{cl}\Gamma(\mathcal A)$ for each $\mathcal{A}$, then any neighborhood of $(R,S)$ w.r.t.~the weak convergence intersects $\Gamma$. To verify this, let $O(R,S)$ be a neighborhood of $(R,S)$ w.r.t.~the weak convergence, then there exists $\tilde{U}(R,S,\mathcal B,\zeta)\in O(R,S)$ over a finite partition $\mathcal B$ as $O(R,S)$ is also open in $\mathcal P_\tau\times\mathcal P_\tau$. Furthermore, the partition $\mathcal B$ can be chosen to refine $\mathcal A$ so that $\operatorname{cl}\Gamma(\mathcal B)\subset\operatorname{cl}\Gamma(\mathcal A)$. As $\tau$-topology is stronger than the weak topology,  a closed set in the $\mathcal P_\tau\times\mathcal P_\tau$ is closed in $\mathcal P\times\mathcal P$, and hence $\operatorname{cl}\Gamma(\mathcal B)\subset \operatorname{cl}_{\tau} \Gamma(\mathcal B)$. That $(R,S)\in\operatorname{cl}_\tau\Gamma(\mathcal B)$ implies that there exists $(R',S')\in\tilde{U}(R,S,\mathcal B,\zeta)\cap\Gamma(\mathcal B)$. By the definition of $\Gamma(\mathcal B)$, we can also find $(\tilde{R},\tilde{S})\in\Gamma$ such that $\tilde{R}(B_i)=R'(B_i)$ and $\tilde{S}(B_i)=S'(B_i)$ for each $B_i\in\mathcal B$, and hence  $(\tilde{R},\tilde{S})\in\tilde{U}(R,S,\mathcal B, \zeta)$. In summary, we have $(\tilde{R},\tilde{S})\in\tilde{U}(R,S,\mathcal B,\zeta)\subset O(R,S)$ and  $(\tilde{R},\tilde{S})\in\Gamma$. Therefore, $\Gamma\cap O(R,S)\neq\varnothing$ and the claim follows.

Next we show that, for each partition $\mathcal A$, 
\begin{align}
\Gamma(\mathcal A)\cap B(P,Q,\eta)\neq\varnothing.
\end{align}
By Eq.~(\ref{eqn:KLDdef}), there exists $(\tilde{P},\tilde Q)$ such that 
$cD(\tilde{P}^{\mathcal A}\|P^{\mathcal A})+(1-c)D(\tilde{Q}^{\mathcal A}\|Q^{\mathcal A})\leq\eta$. For such $(\tilde P, \tilde Q)$, we can construct $(P',Q')\in\Gamma(\mathcal A)$ as 
\begin{align}
P'(F)&=\sum_{i=1}^t\frac{\tilde{P}(A_i)}{P(A_i)}P(F\cap A_i),\nn\\
Q'(F)&=\sum_{i=1}^t\frac{\tilde{Q}(A_i)}{Q(A_i)}Q(F\cap A_i),\nn
\end{align}
for any measurable subset $F\subset\mathcal X$. If $P(A_i)=0$ ($Q(A_i)=0$) and hence $\tilde{P}(A_i)=0$ ($\tilde{Q}(A_i)=0$), as $D(\tilde P^{\mathcal A}\|P^{\mathcal A})<\infty$ ($D(\tilde Q^{\mathcal A}\|Q^{\mathcal A})<\infty$), for some $i$, the corresponding term in the above equation is set equal to $0$. Then $(P',Q')$ belongs to $\Gamma(\mathcal A)$ and also lies in $B(P,Q,\eta)$. The latter is because $D(P'\|P)=D(\tilde{P}^{\mathcal{A}}\|Q^{\mathcal A})$ and $D(Q'\|Q)=D(\tilde{Q}^{\mathcal A}\|Q^{\mathcal A})$: one can verify that any $\mathcal B$ that refines $\mathcal A$ satisfies \[D({P}'^{\mathcal B}\|P^{\mathcal B})=D(\tilde{P}^{\mathcal A}\|P^{\mathcal A}), D({Q}'^{\mathcal B}\|Q^{\mathcal B})=D(\tilde{Q}^{\mathcal A}\|Q^{\mathcal A}).\]

For any finite collection of partitions $\mathcal A_i\in\Pi$ and $\mathcal A\in\Pi$ refining each $\mathcal A_i$, each $\Gamma(\mathcal A_i)$ contains $\Gamma(\mathcal A)$. This implies that 
\[\bigcap_{i=1}^r\left(\Gamma(\mathcal A_i)\cap B(p,q,\eta)\right)\neq\varnothing,\]
for any finite $r$. Finally, the set $\operatorname{cl}\Gamma(\mathcal A)\cap B(P,Q,\eta)$ for any $\mathcal A$ is compact due to the compactness of $B(P,Q,\eta)$, and any finite collection of them has non-empty intersection. It follows that all these sets is also non-empty. This completes the proof.

\section{Proof of Theorem~\ref{thm:mainresult1}}
\label{sec:proofmainresult}
Two lemmas are needed: the first states the optimal type-II error exponent of any level $\alpha$  test for simple hypothesis testing between two known distributions $P$ and $Q$, and the second provides a large deviation bound on $d_k(P,\hat P_n)$.
\begin{lemma}[Stein's lemma~{\cite{Cover2006,Dembo2009}}]
	\label{lem:Steins}
	Let $x^n$ i.i.d.~$\sim R$. Consider the  test between $H_0: R = P$ and $H_1:R=Q$ 
	with $0<D(P\|Q)<\infty$. Given $0<\alpha<1$, let $\Omega^*(n)=(\Omega_0^*(n), \Omega_1^*(n))$ be the optimal level $\alpha$ test such that the type-II error probability is minimized. Then the type-II error probability decreases exponentially at a rate of $D(P\|Q)$ as $n\to\infty$, that is,
	\[\lim_{n\to\infty}-\frac{1}{n}\log Q(\Omega_0^*(n)) = D(P\|Q).\]
\end{lemma}
\begin{lemma}[{\cite{Szabo2015Two,Szabo2016learning}}]
	\label{lem:gamman}
Let $P$, $x^n$, and $\hat P_n$ be defined as in Section~\ref{sec:pre}. Let $k$ be bounded continuous and characteristic, with $0\leq k(\cdot,\cdot)\leq K$. When $x^n~\text{i.i.d.}\sim P$,
	\[\mathbf{P}_{x^n}\left(d_k(P, \hat{P}_{n})> \left({2K}/{n}\right)^{1/2}+\epsilon\right)\leq \exp{\left(-\frac{\epsilon^2n}{2K}\right)}.\]
\end{lemma}

\begin{proof} 
According to Theorem~\ref{thm:MMDmetrize}, $d_k$ metrizes the weak convergence over $\mathcal P$. That $\alpha_{n,m}\leq\alpha$ is clear from Lemma~\ref{lem:Gretton}, and we only need to show that $\beta_{n,m}$ vanishes exponentially as $n$ and $m$ scale. For convenience, we will write the error exponent of $\beta_{n,m}$ as $\beta$.
	
	We first show $\beta\geq D^*$. With a fixed $\gamma>0$, we have $\gamma_{n,m}\leq \gamma$ for sufficiently large $n$ and $m$. Therefore,
	\begin{align}	
	\label{eqn:eqn1}
	\beta&=\liminf_{n,m\to\infty}-\frac{1}{n+m}\log \mathbf{P}_{x^ny^m}(d_k(\hat P_n, \hat Q_m)\leq\gamma_{n,m})\nn\\
	&\geq\liminf_{n,m\to\infty}-\frac{1}{n+m}\log\mathbf{P}_{x^ny^m}(d_k(\hat P_n, \hat Q_m)\leq\gamma)\nn\\
	&\geq\inf_{(R,S):d_k(R,S)\leq\gamma}cD(R\|P)+(1-c)D(S\|Q)\nn\\
	&:=D_\gamma^*,
	\end{align}
	where the last inequality is from the extended Sanov's theorem and that $d_k$ metrizes weak convergence of $\mathcal P$ so that $\{(R,S):d_k(R,S)\leq\gamma\}$ is closed in the product topology on $\mathcal P\times\mathcal P$. Since $\gamma>0$ can be arbitrarily small, we have 
	\[\beta\geq\lim_{\gamma\to 0^{+}}D^*_\gamma,\]
	where the limit on the right hand side must exist as $D^*_{\gamma}$ is positive, non-decreasing when $\gamma$ decreases, and bounded by $D^*$ that is assumed to be finite. Then it suffices to show 
	\begin{align}
	\lim_{\gamma\to0^{+}} D_\gamma^*=D^*.\nn
	\end{align}

	To this end, let $(R_\gamma, S_\gamma)$ be such that  $d_k(R_\gamma,S_\gamma)\leq\gamma$ and $cD(R_\gamma\|P)+(1-c)D(S_\gamma\|Q)=D^*_\gamma$. Notice that $R_\gamma$ and $S_\gamma$ must lie in
	\[\left\{W:D(W\|P)\leq\frac{D^*}c,  D(W\|Q)\leq\frac{D^*}{1-c}\right\}:= \mathcal W,\]
	for otherwise $D_\gamma^*>D^*$. We remark that $\mathcal W$ is a compact set in $\mathcal P$ as a result of the lower semi-continuity of KLD w.r.t.~the weak topology on $\mathcal P$ \cite{VanErven2014RenyiKLD,Dembo2009}. Existence of such a pair can be seen from the facts that $\{(R,S):d_k(R,S)\leq \gamma\}$ is closed and convex, and that both $D(\cdot{\|P})$ and $D(\cdot\|Q)$ are convex functions \cite{VanErven2014RenyiKLD}.
	
	Assume that $D^*$ cannot be achieved. We can write 
	\begin{align}
	\label{eqn:assu}
	\lim_{\gamma\to0^+}D^*_{\gamma}=D^*-\epsilon,
	\end{align}
	for some $\epsilon>0$. By the definition of lower semi-continuity, there exists a $\kappa_W>0$ for each $W\in\mathcal W$ such that 
	\begin{align}
	\label{eqn:geqhalf}
	cD(R\|P)+(1-c)D(S\|Q)\geq cD(W\|P)+(1-c)D(W\|Q)-\frac\epsilon 2	\geq  D^*-\frac\epsilon 2,
	\end{align}
	whenever $R$ and $S$ are both from 
	\[\mathcal S_W=\left\{R:d_k(R,W)<\kappa_W\right\}.\]
	Here the last inequality  comes from  the definition of $D^*$ given in Theorem~\ref{thm:mainresult1}. To find a contradiction, define 
	\[\mathcal S_W'=\left\{R:d_k(R,W)<\frac{\kappa_W}{2}\right\}.\]
	Since $S_W'$ is open and $\bigcup_W\mathcal S_{W}'$ covers $\mathcal W$, the compactness of $\mathcal W$ implies that there exists finite $\mathcal S_W'$'s, denoted by $\mathcal S_{W_1}',\ldots,\mathcal S_{W_N}'$, covering $\mathcal W$. Define $\kappa^*=\min_{i=1}^N\kappa_{W_i}>0$. Now let $\gamma<{\kappa^*}/{2}$ as $\gamma$ can be made arbitrarily small. Since $\bigcup_{i=1}^N \mathcal S'_{W_i}$ covers $\mathcal W$, we can find a $W_i$ with $R_\gamma\in \mathcal S_{W_i}'\subset\mathcal S_{W_i}$. Thus, it holds that \[d_k(S_\gamma,W_i)\leq d_k(S_\gamma,R_\gamma)+d_k(R_\gamma,W_i)<\kappa_{W_i}.\]
	That is, $S_\gamma$ also lies in $\mathcal S_{W_i}$. By Eq.~(\ref{eqn:geqhalf}) we get
	\[cD(R_\gamma\|P)+(1-c)D(S_\gamma\|Q)\geq D^*-\epsilon/2.\]
	However, by our assumption in Eq.~(\ref{eqn:assu}), it should hold that
	\[cD(R_\gamma\|P)+(1-c)D(S_\gamma\|Q)\leq D^*-\epsilon.\]
	Therefore, $\beta\geq D^*$.
	
The other direction can be simply seen from the optimal type-II error exponent in Theorem~\ref{thm:upperbd}. Alternatively, we can use Stein's lemma in a similar manner to the proof of \cite[Theorem~4]{Zhu2018UHT}. Let $P'$ be such that $cD(P'\|P)+(1-c)D(P'\|Q)=D^*.$
	Such $P'$ exists because $0<D^*<\infty$ and $D(\cdot\|P)$ and $D(\cdot\|Q)$ are convex w.r.t.~$\mathcal P$. That $D^*$ is bounded implies that both $D(P'\|P)$ and $D(P'\|Q)$ are finite. We  have 
	\begin{align}
	\beta_{n,m}=&~\mathbf P_{x^ny^m}(d_k(\hat P_n,\hat Q_m)\leq \gamma_{n,m})\nn\\
	\stackrel{(a)}{\geq}&~\mathbf P_{x^ny^m}(d_k(\hat P_n, P')+d_k(\hat Q_m, P')\leq{\gamma_{n,m}})\nn\\
	\stackrel{(b)}{\geq}&~\mathbf P_{x^ny^m}(d_k(\hat P_n, P')\leq\gamma_{n}, d_k(\hat Q_m, P')\leq{\gamma_{m}})\nn\\
	=&~P(d_k(\hat P_n, P')\leq\gamma_{n})\,Q(d_k(\hat Q_m, P')\leq{\gamma_{m}}),\nn
	\end{align}
	where $(a)$ and $(b)$ are from the triangle inequality of the metric $d_k$, and we pick $ \gamma_n=\sqrt{2K/n}(1+\sqrt{-\log\alpha})$, and  $\gamma_m=\sqrt{2K/m}(1+\sqrt{-\log\alpha})$ so that $\gamma_{n,m}>\gamma_n+\gamma_m$. Then Lemma~\ref{lem:gamman} implies  $P'(d_k(\hat P_n,P')\leq\gamma_n)> 1-\alpha$. For now assume that $D(P'\|P)> 0$ and $D(P'\|Q)>0$. We can regard $\{x^n:d_k(\hat P_n,P')\leq\gamma_n\}$ as an acceptance region for testing $H_0:x^n\sim P'$ and $H_1:x^n\sim P$. Clearly, this test performs no better than the optimal level $\alpha$ test for this simple hypothesis testing in terms of the type-II error probability. Therefore, Stein's lemma implies 
	\begin{align}
	\label{eqn:PP}
	\liminf_{n\to\infty}-\frac{1}{n}\log P(d_k(\hat P_n,P')\leq\gamma_n)\leq D(P'\|P).
	\end{align}
	Analogously, we have
	\begin{align}
	\label{eqn:QQ}
	\liminf_{m\to\infty}-\frac{1}{m}\log Q(d_k( \hat{Q}_m,P')\leq\gamma_m)\leq D(P'\|Q).
	\end{align}
	
	Now assume without loss of generality that $D(P'\|P)=0$, i.e., $P'=P$. Then $D(P'\|Q)>0$ under the alternative hypothesis $H_1:P\neq Q$, and Eq.~(\ref{eqn:QQ}) still holds. Using Lemma~\ref{lem:gamman}, we have $P(d_k(\hat P_n,P')\leq\gamma_n)>1-\alpha$, which gives zero exponent. Therefore, Eq.~(\ref{eqn:PP}) holds with $P'=P$.
	
	As $\lim_{n,m\to\infty}\frac{n}{n+m}=c$, we conclude that	\[\beta=\liminf_{n,m\to\infty}-\frac{1}{n+m}\log\beta_{n,m}\leq D^*.\]
	The proof is complete.
\end{proof}

\section{Proof of Theorem~\ref{thm:lowerbd}}
\label{sec:optimality}
\begin{proof}
	Let $P'$ be such that $cD(P'\|P)+(1-c)D(P'\|Q)=D^*$. Consider first $D(P'\|P)\neq0$ and $D(P'\|Q)\neq0$. Since $D^*$ is assumed to be finite, we have both $D(P'\|P)$ and $D(P'\|Q)$ being finite. This implies that $P'$ is absolutely continuous with respect to both $P$ and $Q$, so the Radon-Nikodym derivatives ${dP'}/{dP}$ and ${dP'}/{dQ}$ exist. 
	
	Define two sets 
	\begin{equation}
	\begin{aligned}
	\label{eqn:kldset}
	A_n &= \left\{x^n:D(P'\|P)-\epsilon\leq\frac{1}{n}\log\frac{dP'(x^n)}{dP(x^n)}\leq D(P'\|P)+\epsilon\right\},\\
	B_m &= \left\{y^m:D(P'\|Q)-\epsilon\leq\frac{1}{m}\log\frac{dP'(y^m)}{dQ(y^m)}\leq D(P'\|Q)+\epsilon\right\},
	\end{aligned}
	\end{equation}
	Recall the definition of the KLD: $D(P'\|P)=\mathbf{E}_{x\sim P'}\log(dP'(x)/dP(x))$ and $D(P'\|Q)=\mathbf{E}_{x\sim P'}\log(dP'(x)/dQ(x))$. By law of large numbers,  we have for any given $\epsilon>0$,
	\begin{align}
	\label{eqn:kldsetlargeProb}
	\mathbf P_{x^ny^m}(A_n\times B_m)\geq 1-\epsilon,~\text{for large enough}~n~\text{and}~m,
	\end{align}
	with $x^n$ and $y^m$ i.i.d.~$\sim P'$. 
	
	Now consider the type-II error probability of level $\alpha$ tests. First, for a level $\alpha$ test, we have its acceptance region satisfies
	\begin{align}
	\label{eqn:anyalphatestLargeProb}
	\mathbf P_{x^ny^m}(\A'(n,m))> 1-\alpha,
	\end{align}
	when $x^n$ and $y^m$ i.i.d.~$\sim P'$, i.e., when the null hypothesis $H_0: P=Q$ holds. Then under the alternative hypothesis $H_1:P\neq Q$, we have 
	\begin{align}
	\beta'_{n,m}&= \mathbf P_{x^ny^m} (\mathcal A'_0(n,m))\nn\\
	&\geq\mathbf P_{x^ny^m} (A_n\times B_m\cap\mathcal A'(n,m))\nn\\
	&= \int_{A_n\times B_m\cap\mathcal{A}'(n,m)} dP(x^n)\,dQ(y^m)\nn\\
	&\stackrel{(a)}{\geq}\int_{A_n\times B_m\cap\mathcal A'(n,m)}2^{-n(D(P'\|P)+\epsilon)}2^{-m(D(P'\|Q)+\epsilon)} dP'(x^n)\,dP'(y^m)\nn\\
	&=2^{-nD(P'\|P)-m(D(P'\|Q)-(n+m)\epsilon} \int_{A_n\times B_m\cap\mathcal A'(n,m)}dP'(x^n)\,dP'(y^m)\nn\\
	&\stackrel{(b)}{\geq} 2^{-nD(P'\|P) -mD(P'\|Q)-(n+m)\epsilon}(1-\alpha-\epsilon),\nn
	\end{align}
	where $(a)$ is from Eq.~(\ref{eqn:kldset}) and $(b)$ is due to Eqs.~(\ref{eqn:kldsetlargeProb}) and (\ref{eqn:anyalphatestLargeProb}). Thus, when $\epsilon$ is small enough so that $1-\alpha-\epsilon>0$, we get 
	\begin{align}
	\label{eqn:upp}
	\liminf_{n,m\to\infty}-\frac{1}{n+m}\log\beta'_{n,m}&\leq \liminf_{n,m\to\infty}\frac{1}{n+m}\left(nD(P'\|P)+m(D(P'\|Q)+(n+m)\epsilon\right)\nn\\&=D^*+\epsilon.
	\end{align}
	
	If a test is an asymptotic level $\alpha$ test, we can replace $\alpha$ by $\alpha+\epsilon'$ where $\epsilon'$ can be made arbitrarily small provided that $n$ and $m$ are large enough. Thus, Eq.~(\ref{eqn:upp}) holds too. Finally, since $\epsilon$ can also be arbitrarily small, we conclude that \[\lim_{n,m\to\infty}-\frac{1}{n+m}\log\beta'_{n,m}\leq D^*.\]		
	If $P'=P$, then $A_n$ contains all $x^n\in\mathcal X^n$ and the above procedure gives the same result. 
	
	The same argument also applies the case with $\lim_{n,m\to\infty}\frac{n}{m}=\infty$ and we have
	\[\lim_{n,m\to\infty}-\frac{1}{m}\log\beta'_{n,m}\leq D(P\|Q).\]
\end{proof}

\end{document}